\documentclass[letterpaper,10pt,twocolumn,twoside, conference]{ieeeconf}
\usepackage[letterpaper, left=0.8in, right=0.8in, bottom=0.8in, top=0.8in]{geometry}
\IEEEoverridecommandlockouts                    
\usepackage{times}
\usepackage{setspace}
\usepackage{url}
\usepackage[utf8]{inputenc}
\usepackage[T1]{fontenc}
\usepackage{graphicx}		

\usepackage{amsmath} 
\usepackage{amssymb}  
\usepackage{amsthm}
\usepackage{mathtools}
\usepackage[normalem]{ulem}
\usepackage{paralist}	
\usepackage[space]{grffile} 
\usepackage[dvipsnames,usenames]{xcolor}
\usepackage{bbm}
\usepackage{placeins} 
\usepackage{array}
\usepackage{siunitx}
\usepackage{hyperref}
\usepackage{cleveref}
\usepackage{bm} 
\usepackage{algorithm}
\usepackage{algpseudocode}
\usepackage{cleveref}
\usepackage{cite}

\newcommand\undermat[2]{%
  \makebox[0pt][l]{$\smash{\underbrace{\phantom{%
    \begin{matrix}#2\end{matrix}}}_{\text{$#1$}}}$}#2}

\newtheorem{theorem}{Theorem}
\newtheorem{corollary}{Corollary}
\newtheorem{lemma}{Lemma}
\theoremstyle{definition}
\newtheorem{definition}{Definition}
\theoremstyle{definition}

\theoremstyle{remark}
\newtheorem{remark}{Remark}
\theoremstyle{definition}
\newtheorem{property}{Property}
\theoremstyle{definition}
\newtheorem{assumption}{Assumption}
\theoremstyle{definition}

\theoremstyle{definition}
\newtheorem{problem}{Problem}
\theoremstyle{definition}

\theoremstyle{definition}

\theoremstyle{definition}

\theoremstyle{definition}
\newtheorem{example}{Example}

\newcommand{\mb}[1]{\mathbf{ #1 }}
\newcommand{\bs}[1]{\boldsymbol{#1}}
\renewcommand\b[1]{%
  \ifcat\noexpand#1\relax 
    \bm{#1}
  \else
    \mathbf{#1}
  \fi
}

\newcommand{\Lieg}{\mathcal{G}}

\newcommand{\order}{p}
\renewcommand{\time}{T}

\newcommand{\R}{\mathbb{R}}

\newcommand{\Bez}{\b p}
\newcommand{\BEZ}{\b P}
\newcommand{\z}{\mb{z}} 

\newcommand{\x}{\mb{x}} 
\renewcommand{\u}{\mb{u}} 

\newcommand{\f}{\mb{f}}
\newcommand{\g}{\mb{g}}

\usepackage{balance}

\title{\LARGE \textbf{B\'ezier Reachable Polytopes: Efficient Certificates for \\ Robust Motion Planning with Layered Architectures
}}
\author{Noel Csomay-Shanklin, and Aaron D. Ames%
\thanks{Authors are with the Department of Computing and Mathematical Sciences, California Institute of Technology, Pasadena, CA 91125, USA, Corresponding author: {\tt\small noelcs@caltech.edu}.}%
\thanks{This research is supported by Technology Innovation Institute (TII).}
}

\begin{document}
\maketitle


\begin{abstract}
    Control architectures are often implemented in a layered fashion, combining independently designed blocks to achieve complex tasks. 
    Providing guarantees for such hierarchical frameworks requires considering the capabilities and limitations of each layer and their interconnections at design time.
    To address this holistic design challenge, we introduce the notion of \textit{B\'ezier Reachable Polytopes} -- certificates of reachable points in the space of B\'ezier polynomial reference trajectories. 
    This approach captures the set of trajectories that can be tracked by a low-level controller while satisfying state and input constraints, and leverages the geometric properties of B\'ezier polynomials to maintain an efficient polytopic representation.
    As a result, these certificates serve as a constructive tool for layered architectures, enabling long-horizon tasks to be reasoned about in a computationally tractable manner.
\end{abstract}

\section{Introduction}
Modern control systems overwhelmingly employ layered architectures, wherein independent blocks are combined to achieve complex behaviors \cite{matni_layered}.
Typically, each block is designed in isolation and their interconnections are established in an \textit{ad-hoc} manner.
While this separation enables tractable controller design, achieving joint feasibility between layers is non-trivial.
To create safe and reliable autonomous systems, we need a cohesive theory that considers not only the individual behavior of each block, but also how their interaction effects overall performance and constraint satisfaction. 
In this work, we focus on advancing such a theory for layered architectures that include a trajectory generator (\textit{planner}) and a feedback controller (\textit{tracker}). Specifically, we leverage the geometric properties of B\'ezier polynomials to construct a certificate which enables the connection of such a planner-tracker setup with a high-level decision making layer while maintaining feasibility guarantees, as seen in Figure~\ref{fig:hero}.

The planner-tracker paradigm is is extremely common in robotic systems \cite{kuindersma2016optimization, grandia2023perceptive}, and has theoretical roots in hierarchically consistent control \cite{phi_related}, approximate simulation relations \cite{approx_sim}, and bisimulation \cite{bisimulation_relation}.
In such a framework, the planner ensures feasibility by adjusting the trajectories it generates based on a tracking certificate, i.e. a representation of what the tracking controller can reasonably accomplish.
This concept of layers communicating through achievable performance metrics serves as the foundation for robust motion planning \cite{mpc_book}.
For linear systems, such tracking certificates can be synthesized directly \cite{mayne2005robust}.
For nonlinear systems, generating tracking certificates is a more challenging task, and remains an active area of research. 
One option leverages Hamilton Jacobi reachability analysis to produce tracking upper bounds \cite{fastTrack}.
Alternatively, the linearization of the nonlinear system can be used to get approximate polytopic reachable sets \cite{wu_polytope}.
Depending on the existing system structure, notions of Input to State Stability \cite{sontag2008input} can also be used to constructively produce tracking certificates for nonlinear control systems \cite{csomay2022multirate_extended}. 

\begin{figure}[t!]
    \centering
    \includegraphics[width=\linewidth]{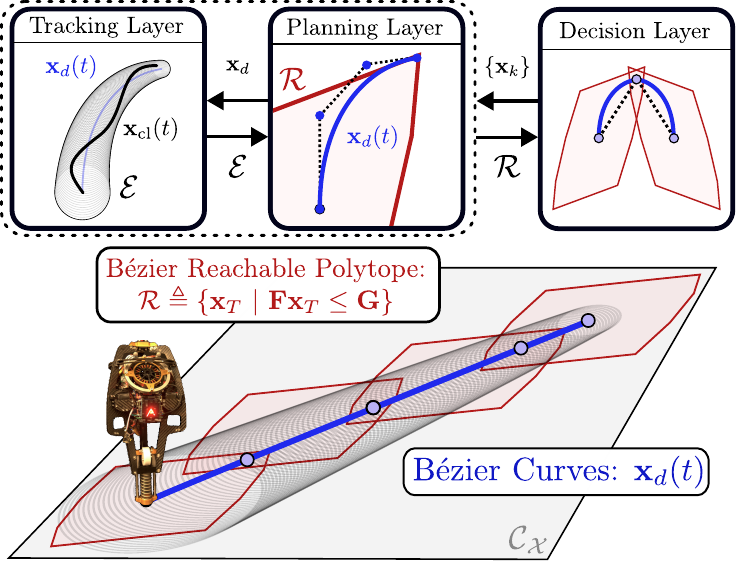}
    \vspace{-4mm}
    \caption{A depiction of the layered architectures investigated in this work, where the reachable set of the combined planning and tracking layers can be represented via a linear inequality in the space of B\'ezier polynomials.}
    \vspace{-6mm}
    \label{fig:hero}
\end{figure}

To extend the notion of guaranteed feasibility to a decision making layer, we require a certificate for the combined planner-tracker model, i.e. a representation of which states can be reached while satisfying state and input constraints.
For discrete systems, this is often done by planning sequences of discrete actions based on motion primitives \cite{donald1993kinodynamic, lavalle2001randomized}. Extending this to arbitrary continuous behaviors often requires solving two point boundary value problems \cite{webb_kinodynamic_2013}, which can be computationally expensive.
Importantly, however, the structure of the planning system can generally be imposed as a part of the design process. We leverage this design degree of freedom by enforcing the planner to generate B\'ezier polynomials; in doing so, we are able to ensure a computationally efficient reachable set representation.  

This paper presents a theory for layered architectures that rely on B\'ezier curves, which have become increasingly popular for motion planning \cite{marcucci_fast_2023, flores_bspline, csomay2022multirate_extended}. We take these ideas further by proving that if the planning layer is parameterized via B\'ezier polynomials, then the space of points which can be reached within a time interval is parameterizeable via a polytope---the \textit{B\'ezier Reachable Polytopes}. We show that these polytopes serve as performance certificates which enable holistic constraint satisfaction guarantees for layered architectures which utilize planning and tracking layers. We demonstrate the use of B\'ezier reachable polytopes in the context of completing long-horizon tasks on a simulated pendulum and experimentally on a physical 3D hopping robot with tight state and input constraints.

\section{Background}

Consider the following nonlinear control system:
\begin{align}
\label{eqn:nonl}
    \dot{\b x} = \b f(\b x, \b u),
\end{align}
with state $\b x \in\mathcal{X}\subseteq\R^N$, input $\b u\in\mathcal{U}\subseteq\R^M$, and whose dynamics $\mb f:\mathcal{X}\times\mathcal{U}\to\R^{N}$ are assumed to be continuously differentiable in their arguments. The system \eqref{eqn:nonl} will be represented as the tuple $\Sigma=\{\mathcal{X}, \mathcal{U}, \b f\}$. Due to the potential complexity of the dynamics $\b f$, directly synthesizing control actions for challenging tasks may be intractable. To address this, control engineers often rely on \textit{planning models}, which serve as template systems that enable desired behaviors to be constructed in a computationally tractable way. These models are defined as:
\begin{definition}
    A system $\Sigma_d=\{\mathcal{X}_d, \mathcal{U}_d, \b f_d\}$ is said to be a \textit{planning model} for a system $\Sigma$ if there exists a surjective mapping $\b \Pi:\mathcal{X}\to \mathcal{X}_d$ and a right inverse $\b \Psi:\mathcal{X}_d \hookrightarrow \mathcal{X}$ such that $\b \Pi\circ\b \Psi = \text{id}_{\mathcal{X}_d}$.
\end{definition}
As the dimensionality of $\mathcal{X}_d$ is typically much smaller than $\mathcal{X}$, there are many possible inverse mappings $\b\Psi$, each of which induce an embedding of the reduced state space $\mathcal{X}_d$ into the full state space $\mathcal{X}$. 
To link a full-order system with a planning model, we must define a feedback controller $\b k : \mathcal{X}\times \mathcal{X}_d\times\mathcal{U}_d\to \mathcal{U}$ which aims to track the states of the planning model. This controller results in the following closed-loop system: 
\begin{align}
\label{eqn:cl}
    \dot{\b x} = \b f(\b x, \b k(\b x, \b x_d, \b u_d))\triangleq \b f_{\text{cl}}(\b x, \b x_d, \b u_d),
\end{align}
which, given any initial condition $\b x_0\in\mathcal{X}$, has continuously differentiable solution $\b x_{\rm cl}:I \to \mathcal{X}$ over some interval $I\subset\R_{\ge 0}$ defined as:
\begin{align*}
    \b x_{\rm cl}(t) \triangleq \b x_0 + \int_0^t \b f_{\textrm{cl}}(\x_{\rm cl}(\tau), \b x_d(\tau), \b u_d(\tau))d\tau.
\end{align*}
A key desired property of this controller is its ability to maintain bounded tracking error:
\begin{definition}
\label{def:invar}
     Let $\Sigma_d$ be a planning model for system $\Sigma$. Given a desired trajectory $\b x_d(\cdot)$, a set-valued function $\mathcal{E}:\mathcal{U}_d \to \mathcal{P}(\mathcal{X})$ is a \textit{tracking certificate} for the system $\Sigma$ if:
     \begin{align*}
        \b x_{\rm cl}(t) \in \b \Psi(\b x_d(t)) \oplus \mathcal{E}(\b u_d(t)),
    \end{align*}
    where $\oplus$ denotes the Minkowski sum.
\end{definition}
\begin{example}
    Let $\Sigma$ represent the closed-loop system of a 3D hopping robot tracking a center of mass velocity command with whole-body model predictive control (MPC). In this scenario, the planning system $\Sigma_d$ is that of a single integrator and the mapping $\b\Pi$ projects the full state space of the hopper into the center of mass planar positions. The function $\b k$ and mapping $\b \Psi$ define the process of MPC, which takes in desired velocity trajectories and produces joint-space trajectories which can be tracked with bounded error via PD control as a function of how much input the planning system applies.
\end{example}
Given a planning model $\Sigma_d$, we will be interested in characterizing the space of all desired trajectories for the system $\Sigma$ which satisfy the following problem:
\begin{problem}
\label{prob:main}
    Consider a compact state constraint set $\mathcal{C}_{\mathcal{X}}\subset \mathcal{X}_d$ and compact input constraint set $\mathcal{C}_\mathcal{U}\subset\mathcal{U}$. Produce trajectories $\b x_d(\cdot)$ which when tracked achieve the following:
    \begin{itemize}
        \item $\b\Pi(\b x_{\rm cl}(t)) \in \mathcal{C}_\mathcal{X}$ for all $t\in I$,
        \item $\b k(\b x_{\rm cl}(t), \b x_d(t), \b u_d(t))\in\mathcal{C}_\mathcal{U}$ for all $t\in I$.
    \end{itemize}
\end{problem}
We will go about solving this problem by appropriately constraining the space of trajectories $\b x_d(\cdot)$, wherein $\b x_d(\cdot)$ will be a design parameter. 
Although planning models can have any system structure (and are useful as long as there exist an appropriate mapping $\b\Psi$ and controller $\b k$), in order to make constructive guarantees we make further assumptions about the planning dynamics. Specifically, consider a nonlinear planning model system with coordinates $\mb q_d\in\R^m$, state 
$\mb x_d = [\mb q_d^\top, \dot{\mb q}_d^\top, \ldots, \mb {q}^{(\gamma-1)}_d {}^\top]^\top \in \R^n$ for some $\gamma \in \mathbb{N}$, and control-affine dynamics of the form:
\begin{align}
    \label{eqn:rom_int}
    \dot{\b x}_d  =\begin{bmatrix} \b 0 & \b I_{n-m} \\ \b 0&\b 0\end{bmatrix} \b x_d + \begin{bmatrix} \b 0 \\ \b f_d(\b x_d)\end{bmatrix} + \begin{bmatrix} \b 0 \\ \b g_d(\b x_d) \end{bmatrix}\b u_d,
\end{align}
where $ \b I_{n-m}$ is an identity matrix of size $n-m$, the $\b 0$ matrices are appropriately sized, $\mb u_d\in\R^m$ is the input, and the drift vector $\mb f_d:\R^n \to \R^m$ and actuation matrix $\mb g_d:\R^n\to\R^{m\times m}$ are assumed to be locally Lipschitz continuous on $\R^n$. We define a dynamically feasible trajectory for such a system as:

\begin{definition}[\textit{Dynamically Feasible Trajectory}]
    Given a time interval $I\triangleq [0,T]$ for $T\in\R_{\ge 0}$, a piecewise continuously differentiable function $\x_d:I\to\R^n$ is a \textit{dynamically feasible trajectory} for $\Sigma_d$ if there is a piecewise continuous function $\mb u_d:I\to\R^m$ such that:
\begin{equation}
    \label{eqn:dynadmistraj}
    \dot {\b x}_d(t) = \f_d(\x_d(t))+\g_d(\x_d(t))\mb u_d(t),
\end{equation}
for almost all $t\in I$.
\end{definition}

In order to design dynamically feasible trajectories $\b x_d(\cdot)$ for $\Sigma_d$, we must reason about integral curves of the planning model dynamics.
To parameterize dynamically feasible trajectories of \eqref{eqn:rom_int} via B\'ezier curves, we assume that the planning system $\Sigma_d$ is fully actuated:
\begin{assumption}
    \label{ass:reldeg}
    We have that $\mb f_d(\mb 0) = \mb 0$ and the matrix $\mb g_d(\mb x_d)$ is invertible for all $\mb x_d\in\mathcal{X}_d$.
\end{assumption} 

\section{B\'ezier Curves}
A curve $\b b:I\triangleq[0,\time] \to \R^m$ for $\time>0$ is said to be a B\'ezier curve \cite{kamermans2020primer} of order $\order\in\mathbb{N}$ if it is of the form:
\begin{align*}
    \b b(t) =   \Bez\b z(t),
\end{align*}
where $\b z:I \to \R^{\order+1}$ is a Bernstein basis polynomial of degree $\order$ and $\Bez \in \R^{m\times\order+1}$ are a collection of $\order+1$ \textit{control points} of dimension $m$. There exists a matrix $\b H\in\R^{p+1\times p+1}$ (as in \cite{csomay2022multirate_extended}) which defines a linear relationship between control points of a curve $\b b$ and its derivative via:
$$\dot{\b b}(t) =  \b \Bez \b H \b z(t).$$
This enables us to define a state space curve $\b B:I\to \R^{n}$:
\begin{align}
\label{eqn:BezDef}
	\b B(t) \triangleq \begin{bmatrix} \b b(t) \\ \vdots \\\b b^{(\gamma-1)}(t)\end{bmatrix} =  \underbrace{%
			\begin{bmatrix}\Bez \\ \vdots \\ \Bez\b H^{\gamma-1}\end{bmatrix}}_{\triangleq \bs{\BEZ}}\b z(t).
\end{align}
The columns of the matrix $\BEZ\in\R^{n\times p+1}$, denoted as $\BEZ_j$ for $j=0,\ldots,p$, represent the collection of $n$ dimensional control points of the B\'{e}zier curve $\b B$ in the state space.
Furthermore, if we take $\b x_d (\cdot)\equiv  \b B(\cdot)$ to represent a desired trajectory of B\'ezier curves, we observe that:
\begin{equation*}
    \dot{\b x}_d = \begin{bmatrix}
        \b 0 & \b I_{n-m} \\ \b 0 & \b 0
    \end{bmatrix}\b x_d + \begin{bmatrix}\b 0 \\ \mb{f}_d(\b x_d)\end{bmatrix} + \begin{bmatrix}\b 0 \\ \mb{g}_d(\b x_d)\end{bmatrix}\mb{u}_d, 
\end{equation*}
for the continuous input signal:
\begin{equation}
    \label{eqn:fl_u}
    \mb{u}_d = \mb{g}_d(\b x_d)^{-1}\Big(\mb{q}_d^{(\gamma)}-\mb{f}_d(\b x_d)\Big).
\end{equation}
Therefore, any B\'ezier curve $\b B(\cdot)$ constructed via \eqref{eqn:BezDef} is a dynamically feasible trajectory for our planning model. As such, we can leverage B\'ezier curves towards the design of trajectories $\b x_d(\cdot)$ satisfying Problem~\ref{prob:main}. B\'ezier curves enjoy a number of desirable properties:

 \begin{property}[Convex Hull \cite{kamermans2020primer}]
	$$\b B(t) \in \text{conv}(\{\BEZ_j\}),~~ j=0,\ldots, \order,~~\forall t\in I.$$
\end{property}
%

%
\begin{property}[Linear Bounding]
\label{prop:linBound}
	For a vector $\b d\in \R^k$ and a matrix $\b C\in \R^{k\times n}$, we have: $$\b C \BEZ_j \le \b d,~~ j=0,\ldots, \order \implies \b C\b B(t) \le \b d,~~\forall t\in I.$$
\end{property}
\begin{proof}
 The convex hull property of B\'{e}zier curves implies that for any $t\in I$ and any row $\b c\in \R^n$ of $\b C$ with corresponding value $d\in \R$ of $\b d$, we may write:
    \begin{align*}
	    \mb c \b B(t) &= \sum_{j=0}^{p} \lambda_j(t) \b c \BEZ_j
    \end{align*}
    for some $\lambda_j(t)\geq 0$ and $\sum_{j=0}^{p} \lambda_j(t) = 1$. 
    Therefore,
    \begin{align*}
	    \mb c \b B(t) &\le \sum_{j=0}^{p} \lambda_j(t)\max_j\b c\BEZ_j = \max_j \b c\BEZ_j \le d,
\end{align*}
	as each $\bm\BEZ_j$ term satisfies $\mb c\BEZ_j \le d$ by assumption.
\end{proof}
\begin{figure}
    \centering
    \includegraphics[width=\columnwidth]{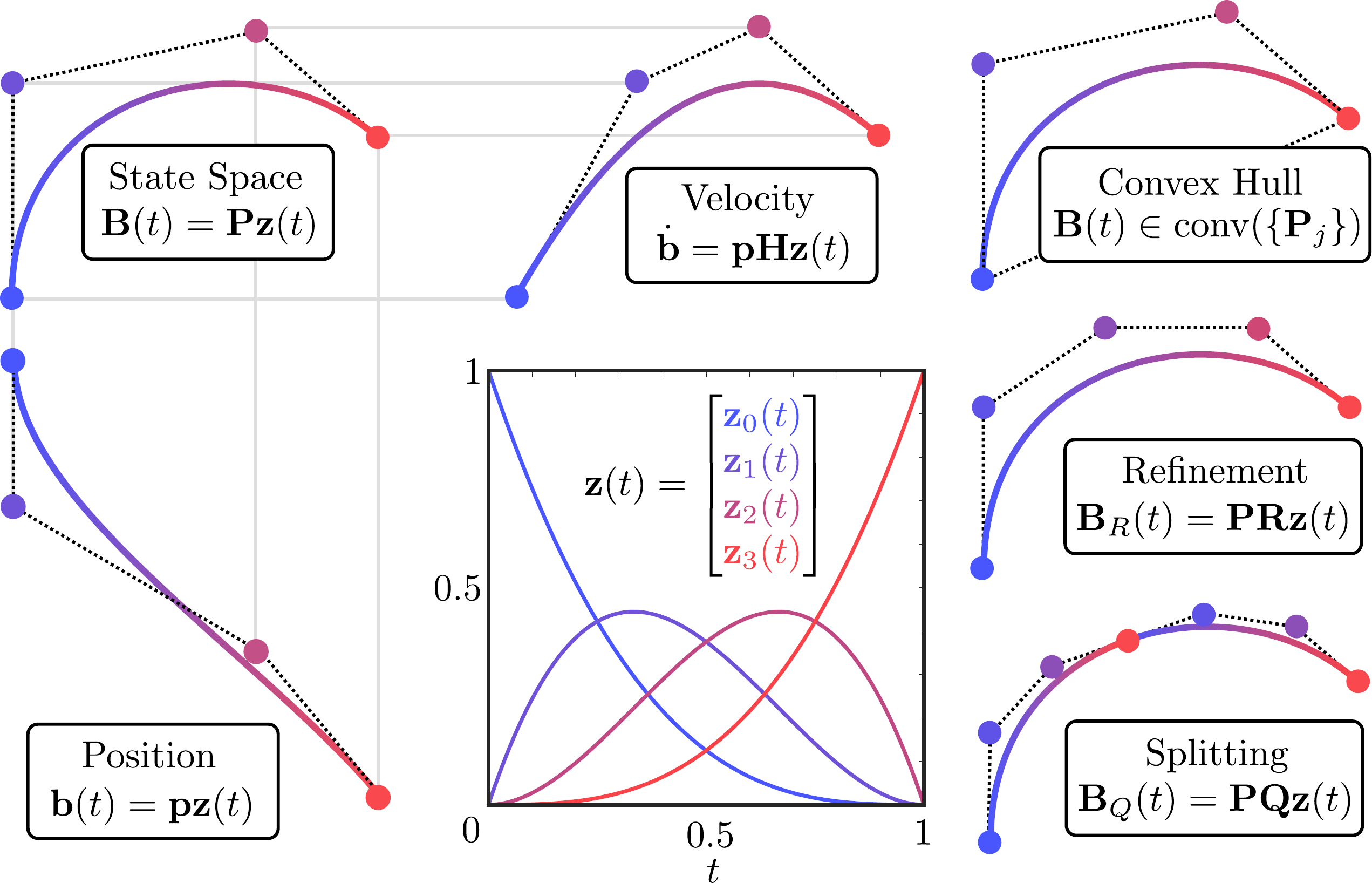}
    \vspace{-5mm}
	\caption{A visual guide to the properties of B\'ezier curves.}
 \vspace{-5mm}
    \label{fig:bez}
\end{figure}

We will specifically be interested in producing B\'ezier curves that connect initial conditions $\mb{x}_d(0)\in\mathcal{X}_d$ and terminal conditions $\mb{x}_d({\time})\in\mathcal{X}_d$ in a fixed time $\time$. 
Given such boundary conditions, a B\'ezier curve $\mb{B}(\cdot)$ which connects them must satisfy the following set of equality constraints:
\begin{align}
	\b b^{(k)}(0) &= \Bez \b H^k \b\z(0) = \b q^{(k)}_d(0) ,~~k=0,\ldots, \gamma-1, \label{eqn:ic}\\
	\b b^{(k)}(\time) &= \Bez \b H^k \b z(\time) = \b q^{(k)}_d(\time),~~k=0,\ldots, \gamma-1.\label{eqn:ec}
\end{align}
These constraints lead to the following Property:
\begin{property}[Boundary Values]
\label{prop:disc_to_bez}
    Given a time $\time > 0$, two points $\mb{x}_0,\mb{x}_{\time}\in \R^n$, and order $p \ge 2\gamma-1$, there exists a matrix $\b D\in\R^{p+1\times 2n}$ such that any curve $\b x_d(\cdot)$ with control points satisfying:
    \begin{align}
    \label{eqn:D_matrix}
        \b \Bez \b D= \begin{bmatrix} \b x_0^\top & \b x_\time^\top\end{bmatrix}
    \end{align} 
    also satisfies $\b x_d(0) = \b x_0$ and $\b x_d(\time) = \b x_\time$.
\end{property}
\begin{proof}
We begin by noting that $\b z(0)=[1~\b 0_{1\times p}]^\top$ and $\b z(\time)= [\b 0_{1\times p}~1]^\top$.
Then, collecting the constraints in \eqref{eqn:ic} and \eqref{eqn:ec} yields:
\begin{align*}
    \b p\begin{bmatrix}\b H^0_0 & \b H^1_0 &\hdots & \b H^{\gamma-1}_0\end{bmatrix}= \b x_0.\\
    \b p\begin{bmatrix}\b H^0_p & \b H^1_p &\hdots & \b H^{\gamma-1}_p\end{bmatrix}= \b x_T.
\end{align*}
where $\b H^i_j$ denotes the $j^{th}$ column of the matrix $\b H$ raised to the $i^{th}$ power. It can be algebraically verified that $\b H$ has the form:
{
\setlength{\arraycolsep}{2.5pt}
\begin{align*}
    \b H^{i}_0 &= \begin{bmatrix} \undermat{i+1}{\star & \cdots & \star} & \undermat{p-i}{0 &\cdots & 0}\end{bmatrix}^\top,~~
    \b H^{i}_p = \begin{bmatrix} \undermat{p-i}{0 &\cdots & 0}& \undermat{i+1}{\star & \cdots & \star}\end{bmatrix}^\top, \\
\end{align*}
}
with nonzero entries $\star$. Taking $\b D\in \R^{p+1\times 2n}$ as:
\begin{align*}
    \b D \triangleq \begin{bmatrix}\b H^0_0 & \b H^1_0 &\hdots & \b H^{\gamma-1}_0 &\b H^0_p & \b H^1_p &\hdots & \b H^{\gamma-1}_p\end{bmatrix},
\end{align*}
in the case that $p\ge 2\gamma-1$ the columns are linearly independent and thus the matrix $\b D$ has full column rank, implying that a solution $\Bez$ exists (but is not unique unless $p=2\gamma-1$).
\end{proof}
\vspace{-3mm}
\begin{remark}
    In the case that $p > 2\gamma-1$, the constraint \eqref{eqn:D_matrix} is under-determined and can be resolved via a least squares solution, allowing for additional cost terms to be optimized.
\end{remark}
Finally, we present one additional property which will be useful in increasing the resolution of B\'ezier curves and reduce the conservatism of their upper bounds. To do this, we introduce the notion of a refinement of the interval $I$ as:
\begin{definition}
\label{def:ref}
    A \textit{$k$-refinement} of an interval $[0,T]$ is a collection of times $\{\time_i\}$ for $i=0,\ldots,k$ and associated intervals $\{[\time_{i-1}, \time_i]\}$ with $\time_{i-1} < \time_i$, $\time_0 = 0$, and $\time_k = T$.
\end{definition}
From this, we can split a B\'ezier polynomial $\b B(\cdot)$ into a sequence of B-splines:
\begin{property}[Splitting \cite{kamermans2020primer}]
Given the control points $\b P$ of a B\'ezier polynomial defined over the interval $I$ and a $k$-refinement of $I$, there exists a collection of matrices $\{\b Q_i\}$ for $i=1,\ldots,k$ 
such that $\b B_Q(t) = \b P \b Q \b z(t)$ satisfies $\b B_Q(t) \triangleq \b B(\time_i + \frac{t}{T}(\time_{i+1} - \time_i))$ for all $t \in I$.
\end{property}




Finally, it will be useful to operate with the (column-wise) vectorized versions of $\Bez$ and $\BEZ$, defined as $\vec\Bez\triangleq\text{vec}(\Bez)\in\R^{m(p+1)}$ and $\vec\BEZ\triangleq\text{vec}(\BEZ)\in\R^{n(p+1)}$. With these new representations, we have the following equivalences:
\begin{align*}
	\vec\BEZ &= \vec{\b H} \vec\Bez \\
	  \vec{\b D}\vec\Bez &= \text{vec}\left(\begin{bmatrix} \b x_0^\top & \b x_T^\top\end{bmatrix}\right)
\end{align*}
with $\vec{\b H}$ and $\vec{\mb D}$ the vectorized versions of $\b H$ and $\b D$, respectively.
With these tools, we will next discuss how to enforce state and input constraints on B\'ezier curves via linear constraints imposed on the control points.

\section{State and Input Constraint Satisfaction}
To begin, we make the following assumption about the constraint sets $\mathcal{C}_\mathcal{X}$ and $\mathcal{C}_\mathcal{U}$:
\begin{assumption}
    The state constraint set is described by $\mathcal{C}_{\mathcal{X}} = \{\b x_d\in\mathcal{X}_d~|~ \b C\b x_d \le \b d\}$ with $\b C \in \R^{k\times n}$ and $\b d \in \R^k$. Furthermore, we have that the input constraint set $\mathcal{C}_\mathcal{U} \triangleq \{\b u\in\R^m~|~\| \b u\|_\infty \le u_{\text{max}}\}$ for $u_{\text{max}}\in\R_{>0}$, i.e., we have a box input constraint.
\end{assumption}
The following constructions can also be performed with a positive diagonal weighting matrix $\mb W\in\mathbb{S}_{\succ 0}^m$ to scale the box constraint on $\mb u$, such that $\Vert\mb{W}\mb{u}\Vert_\infty \leq u_{\text{max}}$. Such constraints are extremely common in robotic systems.
From this point on, $\|\cdot\|$ will represent the $\infty-$norm unless otherwise stated. Given a tracking certificate set, we can define its upper bound ${e}:\mathcal{U}_d \to \R_{\ge 0}$ as:
\begin{align}
    \label{eqn:tracking_ub}
    {e}(\b u_d) \triangleq \sup_{\b e \in \mathcal{E}(\b u_d)}\|\b e\|.
\end{align}
If $\mathcal{E}$ is described as the zero sublevel set of a function that is differentiable with respect to $\b u_d$, then $e(\b u_d)$ is locally Lipschitz with respect to $\b u_d$. Along with this, we assume Lipschitz properties of $\b \Pi$ and $\b \Psi$:
\begin{assumption}
    The functions $\b \Pi$, $\b \Psi$, and ${e}$ are Lipschitz continuous over the domain $\mathcal{C}_\mathcal{X}$ with constants $L_{\b\Pi}$, $L_{\b\Psi}$ and $L_{{e}}$, respectively.
\end{assumption}

The remainder of the section will be devoted to proving the following statement:
\begin{theorem}
\label{thm:BezPoly}
    Let system $\Sigma_d$ be a planning model for system $\Sigma$ with tracking certificate $\mathcal{E}$. There exist matrices $\b F$ and $\b G$ such that any B\'ezier curve $\b B:I\to \mathcal{X}_d$ with control points $\b p$ satisfying:
    \begin{align*}
        \b F \vec{\b p} \le \b G,
    \end{align*}
    when tracked results in the closed loop system satisfying $\b \Pi(\b x_{\rm cl}) \in \mathcal{C}_\mathcal{X}$ and $\b k(\b x_{\rm cl}, \b x_d, \b u_d)\in\mathcal{C}_\mathcal{U}$ for all $t\in I$.
\end{theorem}
Towards this goal, we first show that satisfying input constraints of the tracker can be reformulated as a linear constraint on state and input norms:
\begin{lemma}
    \label{lem:input_constraint}
    Given a reference points $\bar{\b x}_d \in \mathcal{X}_d$, enforcing the constraint:
    \begin{align*}
    \begin{bmatrix} L_{\b k}(1+L_{\b\Psi}) & L_k(1+L_{e}) \end{bmatrix}\begin{bmatrix} \|\b x_d - \bar{\b x}_d\| \\ \|\b u_d\|\end{bmatrix} \le u_{\max} - K(\bar{\b x}_d),
\end{align*}
with $K(\bar{\b x}_d)\triangleq \|{\b k}(\b\Psi(\bar{\b x}_d),\bar{\b x}_d, \b 0)\| + e(\b 0)$ results in input constraints being satisfied, i.e. $\b k(\b x_{\rm{cl}}, \b x_d, \b u_d) \in \mathcal{C}_{\mathcal{U}}$.
\end{lemma}
\begin{proof}
    Observe that the input $\b k$ can be bounded by:
    \begin{align*}
        \|\b k(\b x, \b x_d, \b u_d)\| &\le L_{\b k}(\|\b x - \b \Psi(\b x_d)\| + \|\b \Psi(\b x_d) - \b \Psi(\bar{\b x}_d)\| \\
                     & + \|\b u_d\| + \|\b x_d - \bar{\b x}_d\|) + \|\b k(\b \Psi(\bar{\b x}_d), \bar{\b x}_d, \b 0)\| \\
                     &\le L_k(1+L_{e})\|\b u_d\| + L_{\b k}(1+L_{\b\Psi})\|\b x_d - \bar{\b x}_d\| \\
                     &\hspace{10mm} + \|{\b k}(\b\Psi(\bar{\b x}_d),\bar{\b x}_d, \b 0)\| + e(\b 0).
    \end{align*}
    Rearranging terms yields the desired result.
\end{proof}
We will show in Lemma~\ref{lem:const} that this constraint can be reformulated as a linear inequality constraint on the B\'ezier curve $\b x_d(\cdot)$.
Note that in order for this set to be nonempty, we must have that $u_{\max} - e(\b 0) > 0$. 
This requirement is one of feasibility of the tracking controller -- if not, then the feedback controller applied over the tracking certificate $\mathcal{E}$ is larger than the set $\mathcal{U}$ meaning regardless of desired trajectory there could exist a perturbed state which would violate the input constraints. If this is the case, then the error tracking bound of the low-level controller needs to be improved before proceeding.
%
In order to make a similar claim for state constraints, we present the following claim:
\begin{lemma}
\label{lem:state_constraint}
Enforcing the constraint:
\begin{align}
    \label{eqn:state_constraint}
    \begin{bmatrix} \b C & L_{\b\Pi}L_e\b K\end{bmatrix}\begin{bmatrix}  \b x_d \\ \|\b u_d\|\end{bmatrix} \le \b d - L_{\b\Pi}e(\b 0) \b K
\end{align}
with $\b K\triangleq \sqrt{\textup{diag}(\b C \b C^\top)}$ results in state constraints being satisfied, i.e. $\b \Pi(\b x_{\rm cl}(t)) \in \mathcal{C}_\mathcal{X}$.
\end{lemma}
\begin{proof}
    Recall that applying the controller $\b k$ yields:
\begin{align*}
    \b \Pi(\b x_{\rm cl}) \in \b \Omega\triangleq\Big\{\b\Pi(\b\Psi(\b x_d) + \b v)~|~\b v \in \mathcal{E}(\b u_d)\Big\},
\end{align*}
which holds from Definition~\ref{def:invar}. From \eqref{eqn:tracking_ub}, we continue:
\begin{align*}
    \b\Omega &\subset \Big\{\b\Pi(\b\Psi(\b x_d) + \b v)~|~ \|\b v\| \le e(\b u_d)\Big\} \\
    &\equiv \Big\{\b x_d + \b y~|~\b y = \b\Pi(\b\Psi(\b x_d) + \b v)-\b x_d,\|\b v\| \le e(\b u_d)\Big\} \\
    &\subset \Big\{\b x_d + \b y ~|~ \|\b y\| \le \|\b\Pi(\b\Psi(\b x_d) + \b v)-\b x_d\|, \|\b v\| \le e(\b u_d)\Big\}.
    \end{align*}
    Next, recalling that $\b \Pi \circ \b \Psi(\b x_d) = \b x_d$, we have:
    \begin{align*}
    \b\Omega &\subset \Big\{\b x_d + \b y ~|~ \|\b y\| \le L_{\b\Pi}\|\b v\|, \|\b v\| \le e(\b u_d)\Big\} \\
    &\subset \Big\{\b x_d + \b y ~|~ \|\b y\| \le L_{\b\Pi} e(\b u_d)\Big\}
\end{align*}
Therefore, defining $\rho \triangleq L_{\b\Pi}e(\b u_d)$, we have:
\begin{align*}
    \b \Pi(\b x_{\rm cl}) \in \b x_d(t) \oplus B_{\rho}(0),
\end{align*}
As such, if we can ensure $\b C(\b x_d + \b v) \le \b d$ for all $\b v \in B_{\rho}(0)$, we would have the desired result. Appealing to Lemma 4 in \cite{csomay2022multirate_extended}, we know that this is satisfied if:
    \begin{align*}
        \b C \b x_d &\le \b d - L_{\b\Pi}e(\b u_d) \sqrt{\text{diag}(\b C \b C^\top)},
    \end{align*}
    which is in turn satisfied if:
    \begin{align*}
        \b C \b x_d &\le \b d - \Big(L_{\b\Pi}L_e\|\b u_d\|+L_{\b\Pi}\b e(\b 0)\Big) \sqrt{\text{diag}(\b C \b C^\top)},
    \end{align*}
    which can be rearranged to achieve the desired result.
\end{proof}

Now, we state the following Lemma, which will allow us to reformulate these state and input constraints via linear constraints on the B\'ezier curve:
\begin{lemma}
\label{lem:const}
Given a reference point $\bar{\b x}_d\in\mathcal{X}_d$, a matrix $\b A\in \R^{k\times n+2}$ and vector $\b b\in\R^k$, there exists a matrix $\b L \in \R^{4knm\times n}$ and a vector $\b h\in \R^{4knm}$ such that:
\begin{align*}
\b L\begin{bmatrix}\b x_d\\\b q_d^{(\gamma)}\end{bmatrix}  \le \b h \implies \b A\begin{bmatrix}\b x_d \\ \|\b x_d - \bar{\b x}_d\| \\ \|\b u_d\|\end{bmatrix} \le \b b.
\end{align*}
\end{lemma}
\begin{proof}
    We begin by bounding the term $\b u_d(\cdot)$:
    \begin{align}
        \label{eqn:ud}
        \|\b u_d\| \le \|\mb{g}_d(\b x_d)^{-1}\|\|\mb{q}_d^{(\gamma)}-\mb{f}_d(\b x_d)\|.
    \end{align}
    Taking $\bar{\b x}_d\in\mathcal{X}_d$ to be a reference point in the planning state space, we can bound the first term by:
    \begin{align}
        \label{eqn:g}
        \|\b g^{-1}(\b x_d)\| 
        & \le L_\Lieg\|\b x_d - \bar{\b x}_d\| + \|\b g^{-1}(\bar{\b x}_d)\|,
    \end{align}
    where $L_\Lieg$ is a Lipschitz constant of $\b g^{-1}$ with respect to the $\infty$-norm on $\mathcal{C}_\mathcal{X}$, which is well defined by the local Lipschitz continuity and nonzero assumptions on $\b g$ and the compactness of $\mathcal{C}_\mathcal{X}$.
    Similarly:
    \begin{align}
        \label{eqn:f}
        \|\b q_d^{(\gamma)} - \b f(\x_d)\| \le L_{\b f}\|&\b x_d - \bar{\b x}_d\| + \|\b q_d^{(\gamma)} - \b f(\bar{\b x}_d)\|.
    \end{align}
    Now, let $\b a\triangleq\begin{bmatrix} \b a_1& a_2 & a_3\end{bmatrix}$ be a row of the constraint matrix $\b A$ with $\b a_1\in\R^n$ and $ a_2,a_3\in\R$ and $b\in \R$ the corresponding entry of the vector $\b b$. Substituting \eqref{eqn:g} and \eqref{eqn:f} into \eqref{eqn:ud}, we can construct a quadratic form:
    \begin{align*}
       \begin{bmatrix}a_2 & a_3\end{bmatrix}\begin{bmatrix}\|\b x_d - \bar{\b x}_d\| \\ \|\b u_d\|\end{bmatrix}\le\bm\sigma^\top \b M \bm\sigma_d + \mb N^\top \bm\sigma_d,
    \end{align*}
    where $\bm\sigma_d \triangleq \begin{bmatrix}\|\b x_d - \bar{\b x}_d\| & \| \b q_d^{(\gamma)} - \b f(\bar \x_d)\|\end{bmatrix}^\top$ and:
    \begin{align*}
        \b M = \frac{a_3}{2}\begin{bmatrix}2L_\Lieg L_{\b f} & L_\Lieg \\ L_\Lieg & 0\end{bmatrix},~~
        \b N = \begin{bmatrix}a_3L_{\b f}\|\b g^{-1}(\bar{\b x}_d)\|+a_2 \\ a_2\|\b g^{-1}(\bar{\b x}_d)\|\end{bmatrix}.
    \end{align*}
    
    Next, consider $\widehat{\mb M}$ as the projection of $\b M$ onto the positive semidefinite cone. With this, we can define the function $h:\mathcal{X}_d\times \R^m\to\R$ as:
    \begin{align*}
        h(\b x_d, \b q_d^{(\gamma)}) = \b \sigma_d^\top \widehat{\b M}\b\sigma_d + \b N^\top \b\sigma_d + \b a_1^\top \b x_d.
    \end{align*}
    Because $\b M$ is symmetric, we have that $\widehat{\b M} \preceq \b M$. As such, points in the set $\b\Omega \triangleq \{\b (\b x_d, \b q_d^{(\gamma)})~|~h(\b x_d, \b q_d^{(\gamma)})\le b\}$ satisfy the desired inequality. 
    Next, consider a function $\ell:\mathcal{X}_d\times \R^m\to\R$ of the form:
    \begin{align*}
        \ell(\b x_d, \b q_d^{(\gamma)}) = \b c^\top \b\sigma_d + \b a_1^\top \b x_d,
    \end{align*}
    for some vector $\b c\in \R^2$, along with the following optimization program:
    \begin{subequations}\label{eqn:levelset}
        \begin{align*}
        \delta^* = \sup_{\delta\in \R} \quad & \delta\\
        \textrm{s.t.} \quad & \ell(\b x_d, \b q_d^{(\gamma)}) \le \delta \implies h(\b x_d, \b q_d^{(\gamma)}) \le b 
        \end{align*}
    \end{subequations}
    In general, this set containment problem may be challenging to solve; however, given the specific problem structure this can be solved for in closed form (the details of which can be found in \cite{code}). Then, we have that the set $\b \Lambda \triangleq \{\b (\b x_d, \b q_d^{(\gamma)})~|~\ell(\b x_d, \b q_d^{(\gamma)})\le \delta^*\}\subset\b\Omega$; therefore points in $\b \Lambda$ satisfy the desired constraints.

    Finally, we will show that there exists a matrix $\b L_i\in\R^{4nm\times n+m}$ and a vector $\mb h_i\in\R^{4nm}$ such that:
   $$\b L_i\begin{bmatrix}\b x_d \\ \b q^{(\gamma)}_d\end{bmatrix} \le \mb h_i \Rightarrow \ell(\b x_d, \b q_d^{(\gamma)})\le \delta^*.$$
    Based on the definition of $\b\sigma_d$, the set $\b \Lambda$ is given by:
    \begin{align*}
        \mb c^\top  \begin{bmatrix}\max_i|\b x_d - \bar{\b x}|_i \\ \max_i \left| \b q^{(\gamma)}_d - \b f(\bar{\b x})\right|_i \end{bmatrix}+ \b a_1^\top \b x_d \le \delta^*,
    \end{align*}
    which, taking $\b c^\top = [c_1, c_2]$, is equivalent to:
    {
    \setlength{\belowdisplayskip}{3pt}
    \begin{align*}
        \underbrace{\begin{bmatrix} c_1 &c_1&-c_1&-c_1 \\ c_2 & -c_2 & c_2 & -c_2\end{bmatrix}^\top}_{\triangleq \b F^\top} \begin{bmatrix}\left(\b x_d - \bar{\b x}\right)_i \\ \left(\b q^{(\gamma)}_d - \b f(\bar{\b x})\right)_j \end{bmatrix}+ \b a_1^\top \b x_d \le \b \delta^*,
    \end{align*}
    }
for all row pairs $i\le n$ and $j \le m$ and where $\b\delta^* \triangleq \delta^*\otimes \mb 1$ with $\otimes$ denoting the Kronecker product. 
%
Letting $\b L_i\in \{0,1\}^{4nm\times n+m}$ be matrices capturing the $i,j$ permutations of the scaling matrix $\b F^\top$ above, we can reformulate this as:
    \begin{align*}
        \begin{bmatrix}\b L_1&\b L_2 \end{bmatrix}\begin{bmatrix}\b x_d - \bar{\b x} \\ \b q^{(\gamma)}_d - \b f(\bar{\b x}) \end{bmatrix}+ (\b a_1^\top\otimes\b1) \b x_d \le\b\delta^* ,
    \end{align*}
    which can be further rearranged as:
    \begin{align*}
        \underbrace{\begin{bmatrix}\b L_1 + \b a_1^\top\otimes \b 1&\b L_2 \end{bmatrix}}_{\triangleq\b L_i} \begin{bmatrix}\b x_d \\ \b q^{(\gamma)}_d\end{bmatrix} \le \underbrace{\b\delta^* + \begin{bmatrix}\b L_1&\b L_2 \end{bmatrix} \begin{bmatrix}\bar{\b x}_d  \\ \b f(\bar{\b x}_d)\end{bmatrix}}_{\triangleq \mb h_i}.
    \end{align*}
    Repeating this process for each of the $k$ rows of the constraint matrix $\b A$ yields the desired result.
    \end{proof}
%
The previous Lemma demonstrates that the inequalities on the desired trajectory $\b x_d(\cdot)$ imposed by state and input constraints can be framed as affine constraints on the space of possible trajectories. 
As curves are infinite dimensional objects, traditional trajectory optimizers would generally only approximately enforce these constraints. This is precisely where we see the usefulness of B\'ezier curves -- we can exactly enforce these constraints on the continuous-time curve by reasoning about a discrete, low-dimensional collection of B\'ezier control points (as captured by Property \ref{prop:linBound}).
With this in mind, we are now equipped to prove the main statement of the section:
\begin{proof}[(Proof of Theorem \ref{thm:BezPoly})]
	Enforcing the constraint in Lemma \ref{lem:input_constraint} will result in $\|\b k(\mb{x}_{\rm cl},\b x_d, \b u_d)\|_\infty \le u_{\text{max}}$. Furthermore, from Lemma \ref{lem:state_constraint}, we know that enforcing \eqref{eqn:state_constraint} results in $\b \Pi(\b x_{\rm cl}(t)) \in\mathcal{C}_\mathcal{X}$. Combining these state and input constraints and leveraging Lemma~\ref{lem:const} to produce matrices $\b L_{\b x}, \b L_{\b u}$ and vectors $\b h_{\b x}, \b h_{\b u}$ results in:
    \begin{align}
    \label{eqn:Lambda_h}
        \begin{bmatrix} \b L_{\b u} \\ \b L_{\b x}\end{bmatrix}\begin{bmatrix}\b x_d \\ \b q_d^{(\gamma)} \end{bmatrix}\le \begin{bmatrix}\b h_{\b u} \\ \b h_{\b x}\end{bmatrix}.
\end{align}
	Based on Property \ref{prop:linBound}, we know that if we enforce this constraint on the control points, it will be enforced for the continuous time curve. Therefore, instead we must enforce:
    \begin{align*}
        \begin{bmatrix} \b L_{\b u} \\ \b L_{\b x}\end{bmatrix}\begin{bmatrix}(\BEZ)_j \\ (\Bez\b H^\gamma)_j \end{bmatrix} \le \begin{bmatrix}\b h_{\b u} \\ \b h_{\b x}\end{bmatrix}.
\end{align*}
for $j=0,\ldots,p$. As this imposes linear constraints on the columns of $\Bez$, this can be vectorized and written as:
	$$\b F \vec\Bez \le \b G,$$ where $\b F$ and $\b G$ are appropriate reformulations of \eqref{eqn:Lambda_h} to account for the vectorization. Enforcing this constraint results in state and input constraint satisfaction as desired.
\end{proof}
\label{sec:const}

\section{B\'ezier Reachable Polytopes}
Given the constructions in Section~\ref{sec:const}, there exists an affine inequality that guarantees the existence of a B\'ezier polynomial which results in the closed-loop planner-tracker system satisfying state and input constraints. 
The matrix $\b F$ and vector $\b G$ represent an efficient oracle to check whether B\'ezier curves connecting initial and terminal points satisfy these constraints. 
Combining this affine constraint with Property~\ref{prop:disc_to_bez} allows us to place constraints on the desired boundary conditions of the B\'ezier polynomial -- that is, given an initial condition $\b x_0$, the set characterized by:
\begin{align*}
    \mathcal{F}(\b x_0) = \{\b x_d\in\mathcal{X}_d~|~ \b F \vec{\b{D}}^\dagger \begin{bmatrix} \b x_0^\top & \b x_T^\top \end{bmatrix}^\top \le \b G\},
\end{align*}
represents all terminal conditions for which there exists a feasible B\'ezier polynomial. As such, the set $\mathcal{F}(\b x_0)$ can be thought of as the forward reachable set of the point $\b x_0$. Similarly, given a terminal condition $\b x_T$, the backward reachable set is characterized by:
\begin{align*}
    \mathcal{B}(\b x_T) = \{\b x_0 \in \mathcal{X}_d~|~ \b F \vec{\b{D}}^\dagger \begin{bmatrix} \b x_0^\top & \b x_T^\top \end{bmatrix}^\top \le \b G\}.
\end{align*}
A depiction of the forward reachable set for a pendulum system and a variety of system parameters can be seen in Figure~\ref{fig:gaitTile}. As the error tracking tube $\mathcal{E}$ varies in its dependence on $\b u_d$, the reachable sets change shape to ensure that closed loop system still satisfies the desired constraints.


\begin{figure}[t!]
    \centering
    \includegraphics[width=\linewidth]{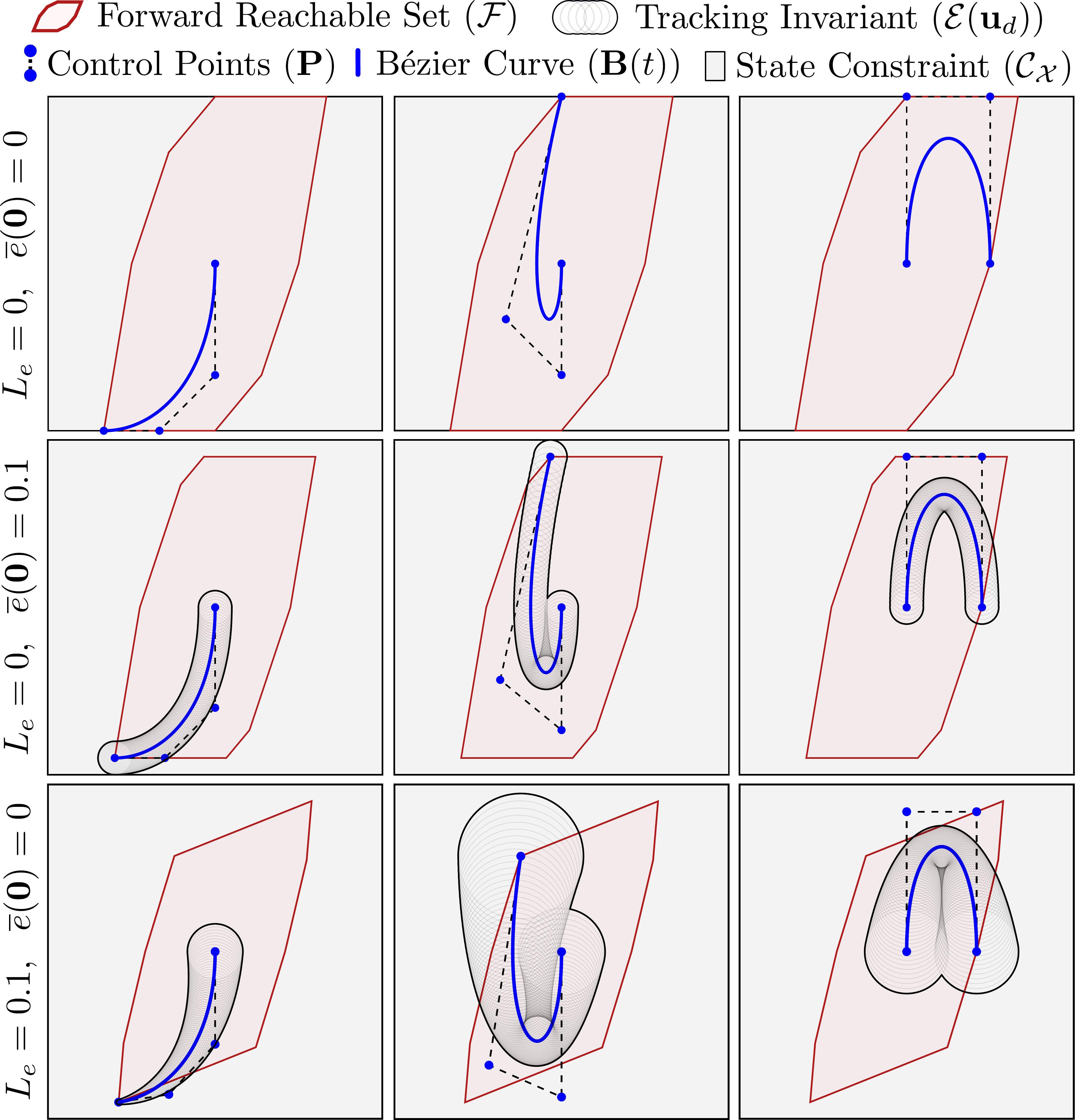}
    \vspace{-4mm}
    \caption{A selection of B\'ezier curves and forward reachable sets. The top row depicts curves with exact tracking, the middle row with a fixed size tracking certificate, and the bottom row with a tracking certificate whose upper bound scales linearly with the planning input $\b u_d$.}
    \vspace{-4mm}
    \label{fig:gaitTile}
\end{figure}

\subsection{Reducing Conservatism}

In the previous discussion, we used a reference point $\bar{\b x}_d$ and bounded the deviation of a trajectory from this point. While this enables tractability, it creates conservatism in the bound as the same reference point was used over the entire trajectory $\b x_d(\cdot)$. To resolve this conservatism, we would like to instead bound the trajectory with a collection of reference points $\{\bar{\b x}_k\}$ spread out over the time interval $[0,T]$. Towards this goal, we leverage the notion of a $k$-refinement of the interval $[0,T]$ from Definition~\ref{def:ref} as well as reference points $\{\bar{\b x}_i\}$ for $i=1,\ldots k$ With these, we can construct a piecewise constant reference trajectory $\bar{\b x}(t) = \bar {\b x}_i$ for $t\in [T_{i-1},T_i)$ with $i=1,\ldots, k$. With this reference trajectory, we have the following:

\begin{corollary}
\label{coll:ref}
	Let system $\Sigma_d$ be a planning model for a system $\Sigma$ with tracking certificate $\mathcal{E}$, and consider a piecewise-constant trajectory $\bar{\b x}(t)$ defined with respect to a $k-$refinement of the interval $[0,T]$. There exist matrices $\widehat{\b F}$ and $\widehat{\b G}$ such that any B\'ezier curve $\b B:I\to \mathcal{X}_d$ with control points $\b p$ satisfying:
    \begin{align}
        \label{eqn:refinedConst}
        \widehat{\b F} \vec{\b p} \le \widehat{\b G},
    \end{align}
    when tracked results in the closed loop system satisfying $\b \Pi(\b x_{\rm cl}(t)) \in \mathcal{C}_\mathcal{X}$ and $\b k(\b x_{\rm cl}(t), \b x_d(t))\in\mathcal{C}_\mathcal{U}$ for all $t\in I$.
\end{corollary}
\vfill
\begin{proof}
    As refinement is linear in the control points, we can leverage the matrices from Theorem~\ref{thm:BezPoly} and right multiply $\b F$ by $\vec{\b Q}_i$, the vectorized version of the refinement matrix $\b Q_i$ for $i=1,\ldots,k$ to produce $\widehat{\b F}$. Taking $\widehat{\b G} = \widehat{\b G}$ yields the desired result.
\end{proof}

\newpage
By enforcing the constraint in \eqref{eqn:refinedConst}, we are able to ensure that the desired trajectory stays close to the piecewise constant reference trajectory, as opposed to a single reference point. This will reduce the conservatism of the bound, but requires increasing the number of constraints needed (and therefore faces of the polytope), demonstrating an obvious tradeoff.
A depiction in the difference in resulting reachable sets can be seen in Figure~\ref{fig:n-step}. When a single points is used, the reachable set indicates the neighborhood around which that reference point can be feedback linearized, potentially requiring significant input over long time horizons. Instead, if we have a sequence of points, we can forward simulate the drift dynamics to produce reference trajectories, whereby the reachable set represents the neighborhood around the trajectory which we can converge to, thereby reducing conservatism. This notion is especially useful when using such reachable sets to represent an MPC layer, which often uses a sequence of reference points to linearize around.

\begin{figure}[t!]
    \centering
    \includegraphics[width=\columnwidth]{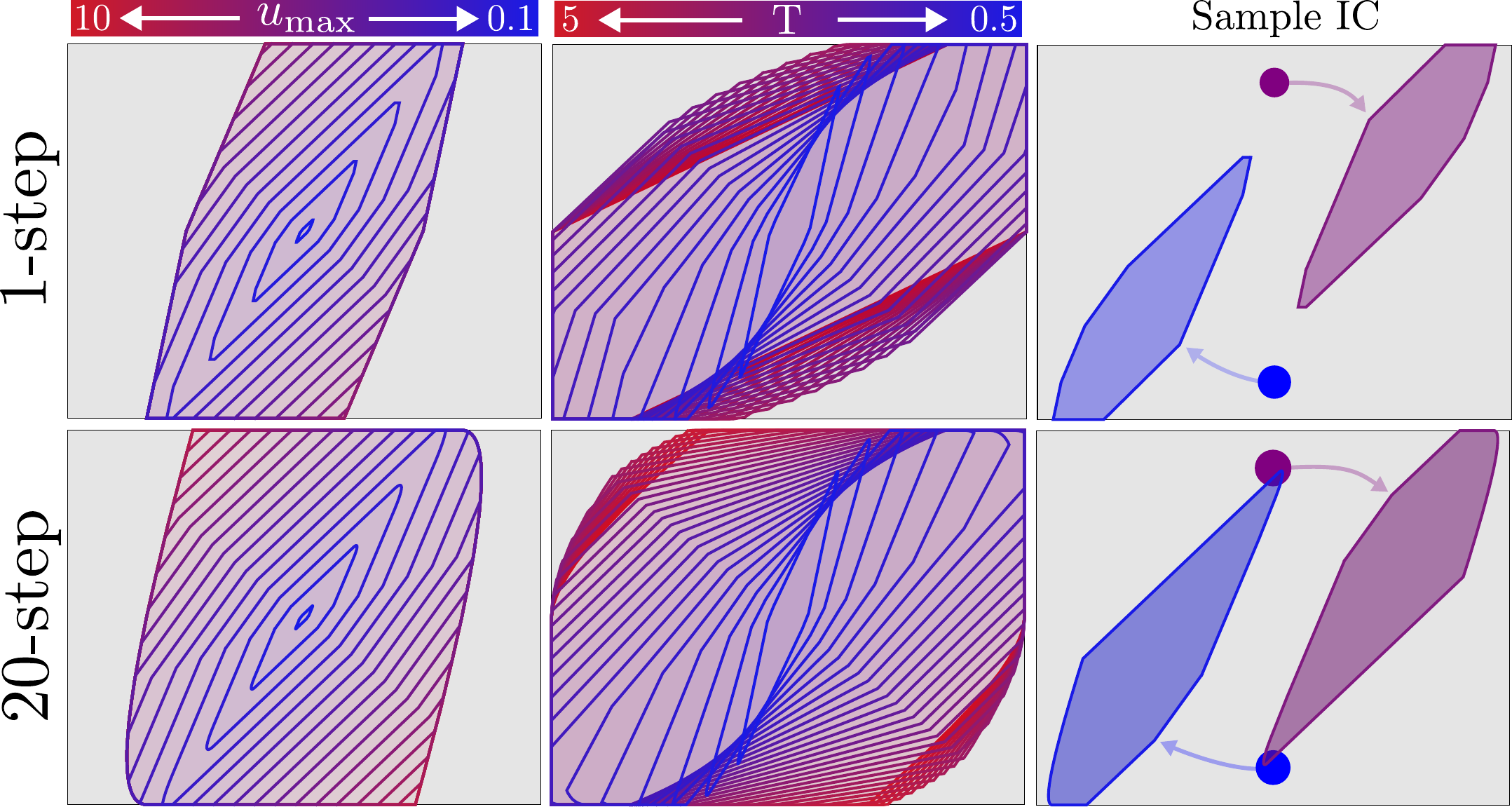}
    \vspace{-5mm}
    \caption{A depiction of the forward reachable sets as a function of system parameters. (Top Row) 1-step reachable sets, as in Theorem~\ref{thm:BezPoly}. (Bottom Row) 20-step reachable sets, as in Corollary~\ref{coll:ref}. (Left Column) Varying the input constraint $\u_{\text{max}}$ (Middle Column) Varying the time horizon T (Right Column) Varying the initial condition.}
    \vspace{-5mm}
    \label{fig:n-step}
\end{figure}

\section{Results}
\subsection{Simulation Results}

We deploy the use of B\'ezier Reachable Polytopes towards the task of swinging up the pendulum. The duration of planning horizon needed to accomplish this task depends highly on how tight the input constraint for the system is.
In this setup, the tracker was taken to be the feedback linearizing controller, and the planner produced trajectories on the pendulum dynamics. 
This planner-tracker was interfaced with a graph-search problem, which samples states uniformly from the state space and connects two vertices $\b v_i, \b v_j\in\mathcal{X}_d$ with an edge if the intersection of their forward and reachable sets were nonempty, i.e. $\mathcal{F}(\b v_i) \cap \mathcal{B}(\b v_j) \ne\emptyset$. This represents a graph of dynamically feasible B\'ezier curves, whereby a suitable B\'ezier curve between two boundary conditions can be found by solving a discrete graph search problem. As seen in Figure~\ref{fig:pendulum}, when the low level input constraints are tight, the graph search has to produce a long sequence of points to achieve pendulum swingup. Instead, if the input constraints are loose, then a nearly direct swingup behavior can be achieved. In this way, we observe that the computational complexity of the decision making layer is imposed by the limitations of the underlying full order system.
The code for this project is available at \cite{code}.

\begin{figure}
    \centering
    \includegraphics[width=\columnwidth]{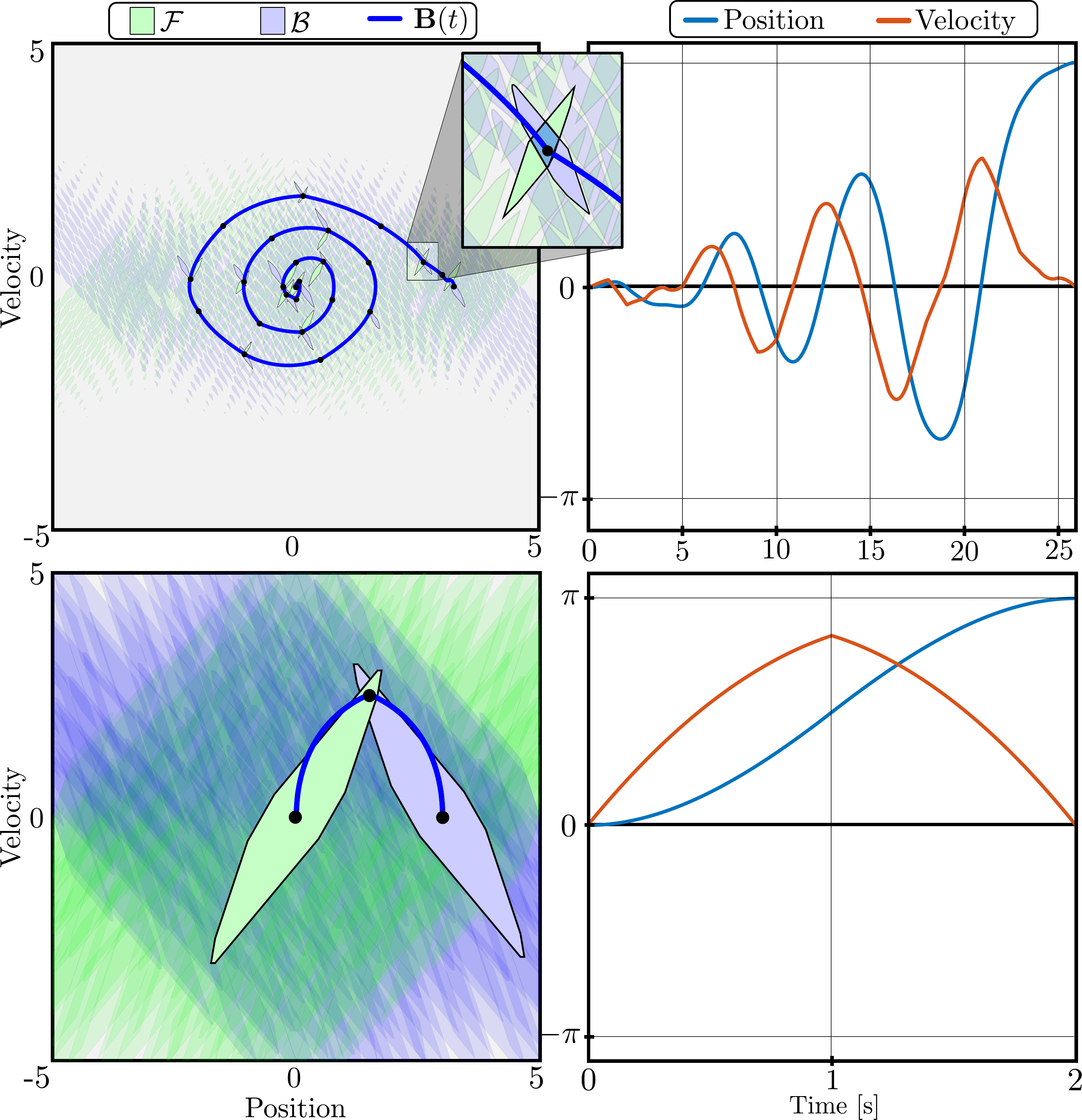}
    \vspace{-5mm}
    \caption{The proposed method applied to the pendulum swingup problem. As the input bounds are tightened from 5 Nm (bottom) to 0.5 Nm (top), the resulting graph search trajectory increases in complexity and length.}
    \vspace{-4mm}
    \label{fig:pendulum}
\end{figure}
\begin{figure*}
    \centering
    \includegraphics[width=\linewidth]{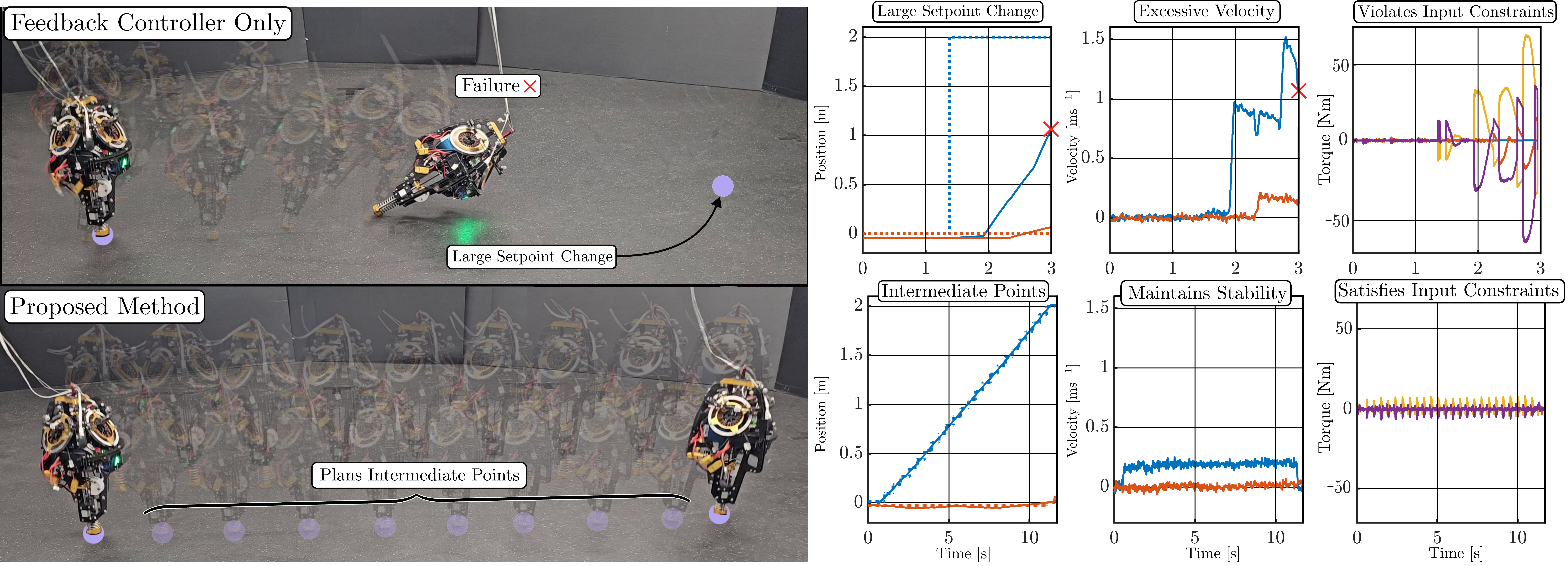}
    \vspace{-6mm}
    \caption{Hardware results on the 3D hopping robot, ARCHER. When commanded to cross the room, a naive decision making layer provides a setpoint to the planner which is outside what is achievable by the low-level system, leading to system failure. Instead, if B\'ezier Reachable Polytopes are used, the decision layer provides a sequence of waypoints to the planner that results in completion of the objective while satisfying state and input constraints.}
    \vspace{-5mm}
    \label{fig:experiment}
\end{figure*}

\subsection{Hardware Results}
We also deploy the B\'ezier Reachable Polytopes framework towards the control of a 3D hopping robot, ARCHER \cite{ambrose_creating_2022}, as seen in Figure~\ref{fig:experiment}. Let $(\b p, q)\in\R^3\times \mathbb{S}^3$ denote the global position and quaternion of the robot, and $(\b v, \b \omega)\in\R^3 \times \mathfrak{s}^3$ the global linear velocity and body frame angular rates. The full state of the robot $\b x\in\mathcal{X}\subset\R^{20}$ contains these values, as well as foot and flywheel positions and velocities. Planning long-horizon tasks for this robot is extremely challenging due to the large number of passive degrees of freedom, tight input constraints, and hybrid dynamics. Separating the path planning problem into a layered architecture consisting of a tracking controller, a planner, and a decision layer enables this task to be split up, whereby behavior can be generated efficiently. 

In this setup, we take the planning model to be a double integrator with state $\b x_d\in \mathcal{X}_d\triangleq \R^4$ and input $\b u_d \in \mathcal{U}_d \triangleq \R^2$. This planning model $\Sigma_d$ can be corresponded with the hopping robot $\Sigma$ by a projection map $\b \Pi:\mathcal{X} \to \R^4$ taken to be the restriction of the full order state to the center of mass $x$ and $y$ positions and velocities and an embedding $\b\Psi$, which is a Raibert-style controller that takes in desired center of mass state and input trajectories and produces desired orientation quaternions as:
\begin{align*}
    q_d(\b x, t) = \b K_{\rm fb} (\b \Pi(x) - \b x_d(t))
\end{align*}
with desired angular rates $\b \omega_d \equiv \b 0$. This desired quaternion is then tracked by a low-level controller $\b k$ as:
\begin{align*}
    \b k(\x, q_d, \b u_d) = -\b K_{\rm p}  \mathbb{I}\textrm{m}(q_{ d}^{-1} q) - \b K_{\rm d}(\b\omega - \b\omega_{ d}) + \b K_{\rm ff} \b u_d,
\end{align*}
which runs at 1 kHz. 
%
As seen in Figure~\ref{fig:experiment}, if only the feedback layer is used, the system fails because the desired setpoint is outside the region of what can be accomplished by the tracking system. Instead, if the proposed method is used, the decision layer can autonomously produce a sequence of points which maintain stability and constraint satisfaction over the task.

\section{Conclusion}
In this work, we introduced the concept of B\'ezier Reachable Polytopes, which provide a representation of the set of points that can be reached by planner-tracker control frameworks. By leveraging the properties of B\'ezier polynomials, we showed that this set can be efficiently represented via a polytopic constraint, enabling computationally tractable long-horizon planning to be achieved. Future work includes developing an abstract theory for such hierarchical control systems and their interconnections. 

\section{Acknowledgements}
The authors would like to thank Andrew Taylor, Preston Culbertson, and Max Cohen for their many fruitful discussions and William Compton for his assistance both with theory with experiments.

\balance
\bibliographystyle{IEEEtran}
\bibliography{main.bib}
\clearpage\newpage

\end{document}